\documentclass{ecai} 

\usepackage{latexsym}
\usepackage{amssymb}
\usepackage{amsmath}
\usepackage{stmaryrd}
\usepackage{amsthm}
\usepackage{booktabs}
\usepackage{enumitem}
\usepackage{graphicx}
\usepackage{color}
\usepackage[rgb,dvipsnames]{xcolor}
\usepackage{tikz}
\usetikzlibrary{calc,decorations.pathreplacing}

\usepackage[disable]{todonotes}
\usepackage{csquotes}
\usepackage{cleveref}
\usepackage{multicol}
\usepackage{subfigure}
\usepackage{xspace}
\usepackage{algorithmicx}
\usepackage[noend]{algpseudocode}
\usepackage{pgfplots}
\usepackage{thm-restate}

\newtheorem{theorem}{Theorem}

\newtheorem{definition}{Definition}
\newtheorem{example}{Example}

\newcommand{\BibTeX}{B\kern-.05em{\sc i\kern-.025em b}\kern-.08em\TeX}

\newcommand{\Omc}{\ensuremath{\mathcal{O}}\xspace}

\newcommand{\ALCQ}{\ensuremath{\mathcal{ALCQ}}\xspace}
\newcommand{\ALCHOIQ}{\ensuremath{\mathcal{ALCHOIQ}}\xspace}
\newcommand{\HornALCHOIQ}{Horn-\ALCHOIQ}
\newcommand{\SROIQ}{\ensuremath{\mathcal{SROIQ}}\xspace}

\newcommand{\NC}{\ensuremath{\mathsf{N_C}}\xspace}
\newcommand{\NR}{\ensuremath{\mathsf{N_R}}\xspace}
\newcommand{\NI}{\ensuremath{\mathsf{N_I}}\xspace}

\newcommand{\JustSet}{\ensuremath{\mathcal{J}}\xspace}
\newcommand{\Just}{\ensuremath{J}\xspace}
\newcommand{\AllJust}{\operatorname{AllJust}}
\newcommand{\Branches}{\ensuremath{\operatorname{Branches}}}
\newcommand{\Fluents}{\ensuremath{\mathcal{F}}\xspace}
\newcommand{\Queries}{\ensuremath{\mathcal{Q}}\xspace}
\newcommand{\SomeIndividuals}{\ensuremath{I}\xspace}

\newcommand{\Jscheme}{\ensuremath{\mathfrak{J}}}

\newcommand{\Xbf}{\mathbf{X}\xspace}
\newcommand{\val}{{\fontsize{8}{8}\selectfont\textsf{val}}\xspace}
\newcommand{\Val}{{\fontsize{8}{8}\selectfont\textsf{Val}}\xspace}
\newcommand{\AxPattern}{\ensuremath{\mathfrak{A}}\xspace}

\newcommand{\PDDLSpec}{\textbf{S}\xspace}
\newcommand{\PDDLdomain}{\textbf{D}\xspace}
\newcommand{\PDDLproblem}{\textbf{P}\xspace}

\newcommand{\Pred}{\ensuremath{\mathsf{N_P}}\xspace}
\newcommand{\PredDer}{\ensuremath{\mathsf{N_{P,D}}}\xspace}
\newcommand{\Obj}{\ensuremath{\mathfrak{O}}\xspace}
\newcommand{\Der}{\mathfrak{D}\xspace}
\newcommand{\act}{\ensuremath{\mathfrak{a}}\xspace}
\newcommand{\state}{\ensuremath{\mathfrak{s}}\xspace}
\newcommand{\Act}{\mathfrak{A}}

\newcommand{\pre}{\frak{pre}\xspace}
\newcommand{\eff}{\ensuremath{\frak{eff}}\xspace}
\newcommand{\del}{\frak{del}\xspace}
\newcommand{\add}{\frak{add}\xspace}
\newcommand{\goal}{\frak{G}\xspace}

\newcommand{\tup}[1]{\langle #1\rangle}
\newcommand{\hide}[1]{}

\newcommand{\pddl}[1]{\textup{\texttt{#1}}}
\newcommand{\owl}[1]{{\fontsize{8}{8}\selectfont \textup{\textsf{#1}}}}

\newcommand{\noteInline}[1]{\todo[inline, color=lime]{#1}}

\newcommand{\patrick}[1]{\todo[inline,color=orange]{Patrick: #1}}

\newcommand{\OP}{\textbf{OP}}
\newcommand{\QueryInter}{\textbf{Q}}
\newcommand{\FluentInter}{\textbf{F}}
\newcommand{\Ostatic}{\ensuremath{\mathcal{O}_s}\xspace}

\newcommand{\exampleInc}[1]{\begin{tabular}{c}
		$A(#1), B(#1), A_C(#1),$\\
		$ A \sqcap B \sqsubseteq C, C \sqcap A_C \sqsubseteq \bot$
	\end{tabular}}

\newcommand{\cand}{\textsf{cnd}}
\newcommand{\detect}[1]{\operatorname{ent}(#1)}	
\newcommand{\rew}{\operatorname{rew}}
\newcommand{\dr}{\operatorname{dr}}

\newcommand{\SingleJust}{\textsc{SingleJust}\xspace}
\newcommand{\Amc}{\mathcal{A}}

\newcommand{\tr}[1]{#1}
\newcommand{\submission}[1]{}

\begin{document}


\begin{frontmatter}


\paperid{1206}


\tr{\title{Planning with OWL-DL Ontologies (Extended Version)}}
\submission{\title{Planning with OWL-DL Ontologies}}


\author[A]{\fnms{Tobias}~\snm{John}\orcid{0000-0001-5855-6632}\thanks{Email: tobiajoh@ifi.uio.no}}
\author[B]{\fnms{Patrick}~\snm{Koopmann}\orcid{0000-0001-5999-2583}\thanks{Email: p.k.koopmann@vu.nl}}

\address[A]{University of Oslo, Gaustadall\'{e}en 23B, 0316 Oslo, Norway}
\address[B]{Vrije Universiteit Amsterdam, De Boelelaan 1105, 1081 HV Amsterdam,
	The Netherlands}


\begin{abstract}
	We introduce ontology-mediated planning, in which planning problems are combined
	with an ontology. Our formalism differs from existing ones in that we focus on a strong
	separation of the formalisms for describing planning problems and ontologies, which are
	only losely coupled by an interface. Moreover, we present a black-box algorithm that supports
	the full expressive power of OWL DL. This goes beyond what existing approaches combining
	automated planning with ontologies can do, which only support limited description logics
	such as DL-Lite and description logics that are Horn. Our main algorithm relies on rewritings
	of the ontology-mediated planning specifications into PDDL, so that existing planning systems
	can be used to solve them. The algorithm relies on justifications,
	which allows for a generic approach that is independent of the expressivity
	of the ontology language. However, dedicated optimizations for computing 
	justifications need to be implemented to enable an efficient rewriting procedure.
	We evaluated our implementation on benchmark sets from several domains.
	The evaluation shows that our procedure works in practice and that tailoring the 
	reasoning procedure has significant impact on the performance.
\end{abstract}

\end{frontmatter}


\section{Introduction}\label{sec:introduction}

Automated planning continues to play a central role in many application domains
such as robotics. Systems for automated planning compute plans for a
\emph{planning domain} which defines a set of available actions that
describe how states of the system can be modified. A \emph{planning problem}
then consists of an initial state, a domain of objects and a goal that needs
to be reached, for which the planner computes a \emph{plan} in the form
of a sequence of actions that would lead from the initial state to a
state satisfying the goal~\cite{GhallbHKMRV1998,Lifschitz1987}.
%

Planning problems classically operate under a \emph{closed world semantics},
meaning that the initial state and all states that the system can go
through are represented as finite first-order interpretations, which requires
a suitable abstraction of the real world. In more complex scenarios, this
assumption is not realistic, as agents have only limited knowledge about the
world, and need to reason over complex domains. Ontologies are a
widely used formalism for describing complex domain knowledge under
\emph{open world semantics}. Here, we assume our knowledge of the world to
be generally incomplete~\cite{DL_TEXTBOOK}.
The ontology defines terminology and background
knowledge relevant for the application scenario, which allows to infer new
information from an incomplete description of the world. 
This also gives further
flexibility when specifying planning problems: for example, the available
actions might depend on a configuration of the robot, whose abilities
can be derived through the technical knowledge defined in the ontology.

Our first contribution is to propose \emph{ontology-mediated planning specifications}
as a new framework for linking planning problems with ontologies, inspired
by ontology-mediated model checking~\cite{OM-PMC}. The
idea is to have an ontology describing static knowledge described in
description logics (DL), the central formalism for specifying ontologies,
together with a specification of a planning problem in PDDL, the classical
language for specifying planning problems. These two are
then linked together through an interface that specifies how to translate
between PDDL and DL. Each state in the planning space has then two perspectives that we also
distinguish syntactically: as ground atoms in PDDL and as DL axioms.
%
By separating the two formalisms
for planning and ontologies,
existing formalizations can easily be integrated,
and planning and ontology experts do not have to learn a new formalism.
Moreover, we achieve a true \emph{separation of concerns}, allowing to
specify planning problems and domain knowledge relatively independently
and in dedicated languages.

Ontology-mediated planning specifications (OMPSs) are strongly related to
extended Knowledge Action Bases (eKABS)~\cite{HaririCMGMF2013,CalvaneseMPS2016}, which extend
classical DL knowledge bases with actions formulated over the
DL signature. One syntactical
difference between eKABs and ontology-mediated
planning specifications is that there is an integration rather than a
separation of concerns. Another difference is that eKABs allow to refer in
both pre- and post-conditions of actions to objects outside of the
domain. Existing techniques for planning with eKABs rely on the idea of
\emph{compilation schemes} that translate the actions in  an eKAB into
a PDDL planning domain, so that a PDDL planning system can be used out
of the shelf. The compilation scheme introduced in~\cite{CalvaneseMPS2016} supports
DLs of the DL-Lite family, while the compilation scheme
introduced in~\cite{BorgwardtHKKNS2022} supports Horn-DLs.
\cite{BorgwardtHKKNS2022} also studies the limitations of such compilation schemes
in general: for DLs such as Horn-\SROIQ, under common complexity theoretic assumptions,
no compilation scheme  that increases the size of the planning domain at most polynomially exists. Moreover, since the approach
relies on Datalog rewritings, it is not applicable to
DLs that are not Datalog-rewritable, such as DLs beyond the Horn fragment.
This means, in this approach,
\emph{disjunction}, unrestricted \emph{negation} and
\emph{value restrictions}, as well as \emph{cardinality restrictions} are not allowed. For instance, it
is impossible to express in a Horn DL that a robot can carry at most five
objects, since this would require a cardinality restriction that is not
available in Horn logics.

Our second contribution is that we use a rewriting technique that is not
restricted to Horn DLs, but can indeed be used for
any DL which is supported by a reasoner, and thus in practice supports
the full ontology language standard OWL DL via the reasoner HermiT~\cite{GlimmHMSW2014}.
Similar to existing approaches, our technique
generates specifications in PDDL to be processed by an off-the-shelf
planning system. However, we do not provide a compilation scheme, which
would only consider the \emph{domain} (i.e. set of actions) of the planning
specification, but we also also take the planning problem
itself into account. In other words, we solve part of the planning problem while
generating the rewritings, and our solution is not \emph{problem-independent}
anymore. Similar to the technique for ontology-mediated model
checking in~\cite{OM-PMC}, we make this possible through the use of
\emph{justifications}, which are minimal consistent sets of
axioms that can be computed with standard software using DL reasoners
as blackbox%
%
~\cite{HST_Algorithm_OWL}. However, additional efforts are needed
as our domains are larger than the ones considered
in~\cite{OM-PMC}. We thus developed new algorithms for computing
justifications that are optimized to the specific form of axiom
sets that are relevant when rewriting ontology-mediated planning
specifications into PDDL.
An evaluation on planning instances from different domains shows that these
optimizations are indeed crucial for making our technique practical,
and on some problems even lead to comparable performances
to those of existing planning techniques optimized for less
expressive DLs.

\tr{%
The framework has been previously introduced in workshop papers~\cite{OMPlanningPLATO,OMPlanningDL},
while the optimizations to make it efficient in practice have not been presented before.
Additional proofs and evaluation results
can be found in the appendix.
}
\submission{%
Additional proofs and evaluation results can be found in the extended
version of the paper~\cite{EXTENDED}.
}

%
%
%

\paragraph{Related Work.}
Our paper relates most to research into
eKABs~\cite{HaririCMGMF2013,CalvaneseMPS2016,BorgwardtHKKNS2022}, and
into ontology-mediated model checking, whose
\emph{ontologized programs} have a similar structure to our
ontology-mediated planning specifications~\cite{OM-PMC}. All
these approaches represent system states using sets of description
logic axioms interpreted under the open-world semantics, on which
actions can only operate by removing or adding statements. There is
also another line of research that considers actions that instead
directly operate on DL interpretations~\cite{DL-ACTIONS},
leading to issues such as decidability of plan existence~\cite{ZarriessClassen2015} and the
ramification problem~\cite{DL-ACTION-RAMIFICATION}.
While~\cite{Milicic2007} studies planning with
such actions, there is to the best of our knowledge so far no practical
system for planning with such actions.

There is also a relation to temporal DLs~\cite{Temp-DLs}
and knowledge bases~\cite{Temporal_KB_2013}, which allow to specify
changing states of a system directly through the DL.
This has for instance been used
in~\cite{Temporal_KB_Car_Application2024}. While there are even
temporal DLs that consider branching
time~\cite{DL-CTL}, the focus here is on describing possible sequences of
states, and not on computing plans to achieve a specific goal.

Describing the behavior of a system and using an ontology to reflect on the state of it to gain more insight is also part of work on \emph{semantically lifted states} \cite{KamburjanKSJG2021}. 
This line of work is not focused on planning but on programming languages for digital twins.
Another 
way
of combining planning and ontologies is using ontologies as a modeling language to describe planning domains \cite{AlsafiVMC007,BalakirskyKKPSG2013,KluschGS2005,LouadahPMTHB2022}.

\section{Preliminaries}

\paragraph{Planning Problems}
We consider the common syntax and semantics of PDDL planning problems 
introduced in
\cite{GhallbHKMRV1998,Lifschitz1987} and described in detail in
\cite{FoxL2003}, with the extension for derived predicates~\cite{HoffmannE2005}.
Let $\Pred$ be a set of \emph{predicate names}, and let $\Obj$ be a set of
\emph{constants}. A \emph{state} \state over $\Obj$ is a finite set of ground atoms over
$\Pred$ and $\Obj$, which we identify with the corresponding first-order interpretation.

We denote a special set $\PredDer\subseteq\Pred$ of predicates called \emph{derived predicates}.
A \emph{derivation rule} is of the form $r=p(\vec{x})\leftarrow\phi(\vec{x})$, where
$\vec{x}$ is a vector of variables, $p\in\PredDer$, and $\phi(\vec{x})$ is a first-order formula
with free variables $\vec{x}$ in which derived predicates occur only positively.
We call $p(\vec{x})$ the head and $\phi(\vec{x})$ the body of the rule. The result of applying such
a rule on a state $\state$ is obtained by adding to $\state$ all $p(\vec{c})$ for which
$\state\models\phi(\vec{c})$, where $\vec{c}$ is a vector of constants with the length of $\vec{x}$. Given
a set $\Der$ of derivation rules and a state $\state$, $\Der(\state)$ denotes the result of applying
the rules in $\Der$ until a fixpoint is reached.

An \emph{action} is a tuple $\act= \tup{\vec{x}, \pre, \eff}$ where $\vec{x}$ is a
vector of variables, $\pre$ is a first-order formula
with free variables from $\vec{x}$, and $\eff=\tup{\add,\del}$, where $\add$ and $\del$ are
finite sets of atoms using only variables from $\vec{x}$ and no derived predicates.
We call $\pre$ the \emph{precondition} of $\act$, and
\eff the \emph{effect}. If $\vec{x}=\tup{}$, $\act$ is called \emph{ground}.
An \emph{instance of $\act$ under $\Obj$} is a ground action $\act$ obtained by replacing all variables from $\vec{x}$
by constants from $\Obj$. let $\sigma: \vec{x} \mapsto \Obj$ be the function describing the replacement.
A ground action is \emph{applicable} on a state \state if
$\state\models\sigma(\pre)$ (where \state is treated as an interpretation), in
which case the result of \emph{applying \act on \state} is defined as
$\state(\act) := (\state \setminus \sigma(\del)) \cup \sigma(\add)$.

A \emph{planning domain} is a tuple $\PDDLdomain=\{\Act,\Der\}$ of a set $\Act$ of actions and a set $\Der$
of derivation rules. A \emph{planning problem} is a tuple $\PDDLproblem = \tup{\Obj, \state_0, \goal}$ with
$\Obj$ a set of constants, $\state_0$ a state over $\Obj$ called the \emph{initial state}, and $\goal$ a first-order
formula called \emph{goal}. A \emph{PDDL planning specification} is now a tuple $\PDDLSpec=\tup{\PDDLdomain,\PDDLproblem}$
of a planning domain \PDDLdomain and a planning problem \PDDLproblem. A \emph{plan} for such a planning specification is a
sequence of states $\state_0\ldots \state_n$ s.t. 1)~$\state_0$ is the initial state from $\PDDLproblem$,
2)~for every $i$, $0\leq i<n$, $\state_{i+1}=\state_i(\act)$, where $\act$ is an instance of an action
from $\Act$ over $\Obj$ that is applicable on $\Der(\state_i)$, and 3)~$\state_n\models \goal$, where $\state_n$ is
treated as an interpretation.



\paragraph{Description Logics and Ontologies} 
A DL ontology is a formalization of domain knowledge based on pair-wise disjoint, countably infinite sets \NC of \emph{concepts names} (unary predicates), \NR of \emph{role names} (binary predicates) and \NI of \emph{individual names} (constants), which are used to build complex expressions called \emph{concepts} that describe sets of individuals.
For the context of this paper, ontologies are sets of \emph{axioms}, which either describe terminological knowledge
(\emph{TBox} axioms) by putting concepts into relation with each other (e.g. by saying that a concept name describes the same 
set of objects than a complex concept) or describe assertional knowledge (\emph{ABox} axioms) by assigning concepts and role names to specific individuals. ABox axioms are of the forms $C(a)$, $r(a,b)$, $\neg r(a,b)$, $a=b$, $a\neq b$, where $C$ is a concept, $a$, $b$ are individual names, and $r$ is a role name. Each (TBox or ABox) axiom can be translated into a sentence in first-order logic, so that we can define entailment
between axioms and of axioms from ontologies based on this translation. Specifically, for an ontology \Omc and an axiom $\alpha$, we 
write $\Omc\models\alpha$ if $\alpha$ is logically entailed by the axioms in $\Omc$. For details on the different concept and axiom constructs, see~\cite{DL_TEXTBOOK}.

Our examples rely on the DL \ALCQ, which is a fragment of OWL~DL, and in which concepts $C$ follow the following syntax rule:
\[
	C ::= \top\mid A\mid \neg C\mid C\sqcap C\mid C\sqcup C\mid
			 \exists r.C\mid\forall r.C\mid{\geq}nr.C\mid{\leq}nr.C,
\]
where $A\in\NC$, $r\in\NR$. TBox axioms in \ALCQ are of the form $C\sqsubseteq D$ or $C\equiv D$, with $C$ and $D$ concepts.
	

\section{Framework}
\subsection{Ontology-Mediated Planning}

We capture our framework formally via \emph{ontology-mediated planning
specifications}, which are inspired by the ontologized programs introduced in~\cite{OM-PMC}.
At the heart of our framework is the notion of
ontology-enhanced states, which
combine a PDDL state with an ontology.

\begin{figure*}
	\newcommand{\OWLcolor}{MidnightBlue}
\newcommand{\Planningcolor}{BrickRed}
\newcommand{\Interfacecolor}{OliveGreen}
\tikzstyle{textNode} = [
	anchor=center,
	]

\tikzstyle{class} = [
	textNode,
	rounded corners,
	minimum height=16pt,
	minimum width=20pt,
	line width=0.75,
	draw
	]
\tikzstyle{individual} = [
	draw=black,
	fill opacity=0.1,
	text opacity=1,
	textNode,
	rounded corners,
	minimum height=16pt,
	minimum width=20pt,
	line width=0.75
	]

\tikzstyle{relation} = [
	line width=0.75pt,
	]

	\resizebox{\textwidth}{!}{

	\centering

	\begin{tikzpicture}[line width=0.75pt]
		\coordinate (origin) at (0,0);

		\draw (origin) node[individual, minimum width=2cm] (auv) {\owl{stackBot}};

		\draw ($(auv) + (2.5,0)$) node[individual] (wp1) {\owl{blockA}};
		\draw ($(wp1) + (1.4,0)$) node[individual] (wp2) {\owl{blockB}};
		\draw ($(wp2) + (1.4,0)$) node[individual] (blockC) {\owl{blockC}};

		{\strut
		Block};

		\draw ($(origin.south) + (3,-0.5)$) node[anchor=north] (TBox)
		{\setlength{\tabcolsep}{2pt} \begin{tabular}{r c l}
				$\owl{PR2}(\owl{stackBot})\qquad
				\owl{Block}(\owl{blockA})$ & &
				$\owl{Block}(\owl{blockB}) \qquad
				\owl{Block}(\owl{blockC})$\\
			    $\owl{blockA} \neq \owl{blockB}$& &
			    $\owl{blockA} \neq \owl{blockC}$ \qquad
			    $\owl{blockB}\neq\owl{blockC}$\\
				$\owl{PR2}\sqsubseteq\owl{Robot}\sqcap
				{\leq}2\owl{holds}.\owl{Block}$ & &
				$\owl{PR2} \sqcap {=}2\owl{holds}.\owl{Block}
				\sqsubseteq\owl{FullHands}$
		\end{tabular}};

		\draw ($(auv) + (2.25, 2.75)$) node[textNode] (query)
{$\models \owl{FullHands} (\owl{stackBot}) $};

		\draw ($(query) + (-8,0)$) node[textNode] (PDDLquery)
		{\pddl{fullHands(stackBot)}};

		\draw ($(auv) + (-5.75,1.75)$) node[textNode] (fluent)
		{\pddl{holds(stackBot, blockB)}};
		\draw ($(fluent) + (0,-0.5)$) node[textNode] (fluentA)
		{\pddl{holds(stackBot, blockA)}};

		\draw ($(fluentA) + (0,-1)$) node[textNode] (fluent2)
		{\pddl{on(blockB, blockA)}};
		\draw ($(fluent2) + (0,-0.5)$) node[textNode] (fluent3)
		{\pddl{onTable(blockC)}};

		\coordinate (onPath1) at ($(auv.north) + (2, 1.25)$);
		\coordinate (onPath2) at ($(onPath1) + (0.9, 0)$);

		\draw[->, rounded corners, relation, color=\Interfacecolor]
		($(auv.north) + (0.75, 0)$)
		--
		(onPath1) -- (onPath2) -- (wp2);

		\coordinate (onPath3) at ($(auv.north) + (0, 0.75)$);
		\coordinate (onPath4) at ($(onPath3) + (1.75, 0)$);

		\draw[draw=white,double distance=\pgflinewidth,ultra thick]
		(onPath3) -- (onPath4);
		\draw[->, rounded corners, relation, color=\Interfacecolor]
		($(auv.north) + (-0.75, 0)$) --
		(onPath3) -- (onPath4) -- (wp1);

		\coordinate (help6) at ($(auv) + (2.45, 0)$) ;
		\draw[color=black] (help6 |- fluent) node (relationLabel) {\owl{holds}};
		\coordinate (help14) at ($(auv) + (0.75, 0)$) ;
		\draw[color=black] (help14 |- fluentA) node (relationLabelA)
		{\owl{holds}};


%
%

		\draw[-latex, draw=\Interfacecolor] (fluent) --node[above]
		(Fvertical) {\phantom{$F$}} ($(relationLabel.west) + (-0.25, 0)$);
		\draw[-latex, draw=\Interfacecolor] (fluentA) --
		($(relationLabelA.west) + (-0.25, 0)$);
		\draw[-latex, draw=\Interfacecolor] (query.west) -- node[above]
		(Svertical) {\phantom{$S$}} (PDDLquery);


		\coordinate (help0) at (TBox.south east |- wp2.north east);
		\coordinate (help1) at ($(help0) + (0.2, 0.25)$);
		\draw[decoration={brace, raise=-5pt, amplitude=5pt}, decorate, draw=\OWLcolor]
		(help1) -- node[right=-5pt] {\footnotesize\begin{tabular}{c}
				static\\
				part of\\
				ontology \\
				(\Ostatic)
		\end{tabular}} (help1 |- TBox.south east);

		\coordinate (help2) at (help1 |- query.north);
		\coordinate (help3) at (help1 |- query.south);
		\draw[decoration={brace, raise=-5pt, amplitude=5pt}, decorate, draw=\OWLcolor]
		(help2) -- node[right=-5pt] {\footnotesize\begin{tabular}{c}
				ontology\\
				queries \\
				(\Queries)
		\end{tabular}} (help3);

		\coordinate (help4) at ($(help3) + (0, -0.3)$);
		\coordinate (help5) at ($(help1) + (0, 0.1)$);
		\draw[decoration={brace, raise=-5pt, amplitude=5pt}, decorate, draw=\OWLcolor]
		(help4) -- node[right=-5pt] {\footnotesize \begin{tabular}{c}
				dynamic\\
				part of\\
				ontology\\
				(\Fluents)
		\end{tabular}} (help5);

		\coordinate (help9) at (fluent.west |- PDDLquery.north);
		\coordinate (help10) at (fluent.west |- PDDLquery.south);
		\draw[decoration={brace, mirror, raise=5pt, amplitude=5pt}, decorate,
		draw=\Planningcolor]
		(help9) -- node[left=5pt] {\footnotesize \begin{tabular}{c}
				query-\\
				atoms
		\end{tabular}} (help10);

		\draw[decoration={brace, mirror, raise=5pt, amplitude=5pt}, decorate,
		draw=\Planningcolor]
		(fluent.north west) -- node[left=5pt] {\footnotesize \begin{tabular}{c}
				mapped\\
				atoms
		\end{tabular}} (fluentA.south west);

		\draw[decoration={brace, mirror, raise=5pt, amplitude=5pt}, decorate,
		draw=\Planningcolor]
		(fluent.west |- fluent2.north) -- node[left=5pt] {\footnotesize
		\begin{tabular}{c}
				atoms\\
				outside\\
				mapping
		\end{tabular}} (fluent.west |- fluent3.south);

		\coordinate (help11) at ($(query.north west) + (-1.5, 0)$);
		\draw[decoration={brace, raise=10pt, amplitude=5pt}, decorate, draw=\Planningcolor]
		(fluent.west |- query.north) -- node[above=15pt]
		{\textbf{planning perspective}} (fluent.east |- query.north);
		\draw[decoration={brace, raise=10pt, amplitude=5pt}, decorate, draw=\OWLcolor]
		(help11) -- node[above=15pt] {\textbf{DL perspective}} ($(help2) + (-0.5, 0)$);

		\coordinate (help7) at ($(fluent.east |- query.north) + (0.2, 0)$);
		\coordinate (help8) at ($(help11) + (-0.2, 0)$);
		\draw[decoration={brace, raise=10pt, amplitude=5pt}, decorate, draw=\Interfacecolor]
		(help7) -- node[above=17pt] (interface) {\textbf{interface}} (help8);

		\draw (interface |- Svertical) node {\QueryInter};
		\draw (interface |- Fvertical) node {\FluentInter};

		\coordinate (help12) at (help11 |- help1);
		\draw[dotted, draw=\OWLcolor, line width=1.25pt] ($(help1) + (-0.5, 0.05)$) -- ($(help12) + (-0.0, 0.05)$);

		\coordinate (help13) at (help11 |- help4);

		\draw[draw=\OWLcolor,double] ($(help4) + (-0.5, 0.15)$) -- ($(help13) + (-0.0, 0.15)$);

	\end{tikzpicture}
}
		\vspace{-15pt}
	\caption{Example of ontology based planning. The interface maps ontology
		queries to planning predicates and atoms in the planning perspective to
		ABox atoms.
		The static part of the ontology contains information about instances
		(ABox) as
		well as general axioms (TBox). The connections between
the two perspectives via
		the fluent (\FluentInter) and query (\QueryInter) interface are shown in green.}
	\label{fig:overview}
\end{figure*}
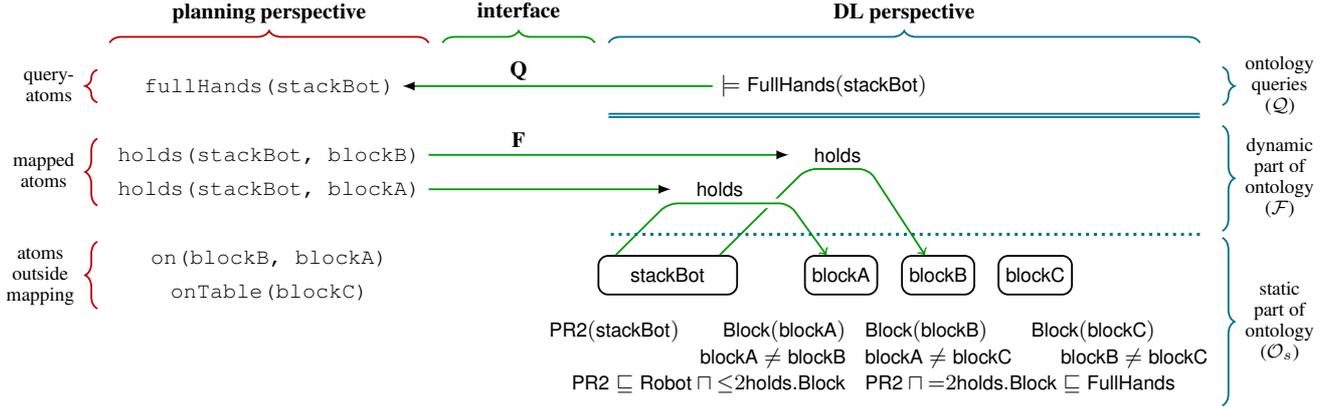

\begin{definition}[Ontology-Enhanced State]
 An \emph{ontology-enhanced state} is a tuple $q=\tup{\state_q,\Omc_q}$, where $\state_q$
is a set of atoms called the \emph{planner perspective of $q$}, and $\Omc_q$ is
a set of DL axioms called the \emph{DL perspective of $q$}.
\end{definition}
While the planner has only direct access to the
planner perspective, which is a PDDL state and allows
predicates of arbitrary arity, DL semantics and reasoning apply to the
DL perspective of the state, which is represented using DL syntax axioms.
The interface, which we define next, ensures that the two perspectives are compatible:
the axioms of the DL perspective
are then determined by the atoms in the planner perspective.
There is also a static component in the DL perspective,
which describes
state-independent information such as class definitions, general domain
knowledge, and the context of the planning problem. The planner
can access derived information from the DL perspective
using query predicates.
Before we give the formal
definition, we illustrate this idea with an example.

\begin{example}\label{example:overview}
An example of an ontology-enhanced state is depicted in
Figure~\ref{fig:overview}.
The scenario is inspired by the classical blocksworld planning domain. In
contrast to the classical problem in which the robot has only one hand, we use an
ontology to specify types of robots and their number of hands. In
the
example, the stacking robot is a \owl{PR2} robot~\cite{BohrenRGMP2011} that can hold
two blocks at a time, and if it holds two blocks,
it becomes an instance of
\owl{FullHands}. While relatively simple, those cardinality constraints
already go beyond the expressivity of \HornALCHOIQ, the most expressive DL
currently supported by existing
implementations for eKABs (see Section~\ref{sec:introduction}).
The planner perspective is shown on the left, the DL perspective on the right, and the interface in the
middle. If the atom \pddl{holds(stackBot, blockA)} becomes true in the
planner perspective, this
is reflected in the DL perspective by an axiom
expressing a corresponding relation between the two
individuals \owl{stackBot} and \owl{blockA}. Using the static
ontology, we can infer that \owl{stackBot} is an instance of the concept
\owl{FullHands}, because it stands in the \owl{holds} relation with
two different blocks. Consequently,
\owl{FullHands(stackBot)} is entailed. The query predicate
\pddl{fullHands} translates to a query for instances
of the concept \owl{FullHands}. Because in the DL perspective,
\owl{stackBot} is an instance of \owl{FullHands},
\pddl{fullHands(stackBot)} becomes true in the
planner perspective of the state.
\end{example}

To implement this mechanism, the interface consists of two components, the
fluent interface and the query interface.
\begin{definition}[Interface]\label{def:interfaces}
A \emph{fluent interface} is a partial inverse-function 
$\FluentInter:\Pred\cup\Obj\rightharpoonup\NC\cup\NR\cup\NI$ that assigns unary
predicates to concept names, binary predicates to role names
and constants to individual names.
We lift $\FluentInter$ to
atoms by setting
$\FluentInter(P(t_1,\ldots,t_n))=\FluentInter(P)\big(\FluentInter(t_1),\ldots,
\FluentInter(t_n)\big)$ if it
is defined.

A \emph{query interface} $\QueryInter$ is a set of \emph{query specifications}
which are expressions of the form
\[p(x_1: C_1,\ldots, x_n: C_n) \leftarrow \Phi(x_1,\ldots,x_n), \qquad\text{where}
\]
\begin{enumerate}[label=\textbf{Q\arabic*}]
 \item $p\in \PredDer$ is the \emph{query predicate},
 \item $x_1$, $\ldots$, $x_n$ are the \emph{query variables}
 \item each $C_i$ is a DL concept specifying the \emph{static type} of $x_i$,
 \item the \emph{query} $\Phi(x_1,\ldots,x_n)$ is a set of ABox axioms
 using variables $x_1$, $\ldots$, $x_n$ as place holders for individual names.
\end{enumerate}
\noteInline{schould we drop the static types?}
\patrick{Relevant for practicality, as I see it.}
\end{definition}
A query specification defines a query predicate $p$ that can be used by the
planner to access the state of the DL perspective. A query $p(x_1,\ldots,x_n)$
translates to a complex query $\Phi(x_1,\ldots,x_n)$ that is evaluated on the ontology
perspective. The \emph{static types} $C_1$, $\ldots$, $C_n$ specify the
range of the variables \emph{independently of the current state}, which
makes the definition more transparent and allows for a more efficient
implementation of our planning approach. For such a query specification,
we call a vector $\tup{a_1,\ldots,a_n}$ of individual names a \emph{candidate}
in an ontology $\Ostatic$ if $\Ostatic\models C_i(a_i)$ for $1\leq i\leq n$.
In our example, we would define
\[
 \pddl{fullHands}(x: \owl{Robot})\leftarrow \owl{FullHands}(x),
\]
but static types and concepts can also be more complex.

We have now all ingredients to define ontology-mediated planning specifications.

\begin{definition}
    An \emph{ontology-mediated planning specification} (OMPS) is a tuple
    $\OP=\tup{\PDDLSpec,\Ostatic,\FluentInter,\QueryInter}$, where
    \begin{itemize}
        \item \PDDLSpec is a PDDL planning specification ,
        \item $\Ostatic$ is an ontology called the \emph{static ontology},
        \item $\FluentInter$ is a fluent interface,
        \item $\QueryInter$ is a query interface.
    \end{itemize}
\end{definition}

An OMPS determines which ontology-enhanced states are compatible: a
state
$q=\tup{\state_q,\Omc_q}$ is \emph{compatible} to an OMPS
$\OP=\tup{\PDDLSpec,\Ostatic,\FluentInter,\QueryInter}$, where
$\Der$ are the derivation rules in $\PDDLSpec$, iff:
\begin{enumerate}[label=\textbf{C\arabic*},leftmargin=*]
 \item\label{comp:state} $\state_q$ is a set of atoms over predicates and constants
occurring in $\PDDLSpec$;
 \item\label{comp:static} $\Ostatic\subseteq\Omc_q$;
 \item\label{comp:fluents} for every $\alpha\in \Der(\state_q)$ for which
$\FluentInter(\alpha)$ is defined;
$\FluentInter(\alpha)\in\Omc_q$
 \item\label{comp:minimal} $\Omc_q$ only contains axioms required
due to Conditions~\ref{comp:static} and~\ref{comp:fluents};
 \item\label{comp:query}
 for every query specification
 $p(x_1: C_1,\ldots, x_n: C_n) \leftarrow \Phi(x_1,\ldots,x_n)$ and
 every candidate $\tup{a_1,\ldots,a_n}$ in $\Ostatic$ s.t.
 $\Omc_q\models\Phi(a_1,\ldots,a_n)$ and $\FluentInter^-(a_i)$ is defined for all
 $1\leq i\leq n$, it is required that $p(\FluentInter^-(a_1),\ldots,\FluentInter^-(a_n))\in\state_q$.
\end{enumerate}

\newcommand{\ext}{\textsf{ext}}

Every planning state $\state$ without query predicates can be uniquely extended to an
ontology-enhanced state $\ext(\state,\OP)$ compatible with $\OP$ where $\state$ is
extended only by query predicate atoms.

It remains to define the semantics of actions and plans on OMPSs. Fix
an OMPS
$\OP=\tup{\PDDLSpec,\Ostatic,\FluentInter,\QueryInter}$ with derivation rules $\Der$. Let $\act$ be a ground action
with precondition $\pre$ and effect $\eff=\tup{\add,\del}$. Let $q$ be an
ontology-enhanced state. We say that $\act$ is \emph{applicable} on $q$ iff
$\Der(\state_q)\models\pre$. The result of \emph{applying $\act$ on $q$}
is then denoted
by $q(\act)$ and defined as $q(\act)=\ext(\state(\act),\OP)$.
We can now define
\emph{plans for \textbf{OP}} similarly as we did for planning specifications:
Namely, a plan is a sequence $\act_1\ldots \act_n$ of actions that generates a
sequence
$q_0q_1\ldots q_n$ of ontology-enhanced states s.t.
\begin{enumerate}[label=\textbf{P\arabic*},leftmargin=*]
 \item $q_0=\ext(\state_0,\OP)$, where $\state_0$ is the initial state of \PDDLSpec,
 \item for each $i\in\{1,\ldots,n\}$: $\act_i$ is applicable on $q_{i-1}$,
    $q_{i}=q_{i-1}(\act_i)$, and $q_i$ is consistent, and
 \item $\Der(\state_{q_n})\models G$, where 
 $G$ is the goal of the planning
problem.
\end{enumerate}
%
%
%

\todo[inline]{TELL and ASK in classical function view}
\hide{
\begin{itemize}
	\item explain structure of fluents + hooks

	--> correspond to TELL and ASK in classical functional view

	\item mention: need to restrict queries (not the usual ECQs)

	\item formal definition of queries (hooks): conjunction of

	--> semantic of query: can we make this similar to \cite{BorgwardtHKKNS2022}?

	\item formal definition of (our version) of eKABs

\end{itemize}
}

\subsection{Planning Procedure}

OMPS are strongly related to eKABs, the main differences being: 1)~eKABs do not explicitly
define an interface, and directly integrate actions and DL axioms, and
2)~eKABs allow for queries and effects with existential quantification,
while our query predicates always translate to queries in which every variable is
mapped to a query variable. eKABs without existentially quantified variables can
thus easily translated into OMPSs, while OMPSs that only use unary and binary predicates
in the planning component can be
translated into eKABs. In the latter case we can use existing techniques for
planning with eKABs, given that they support the expressivity of the ontology part.
There are generally two techniques for planning of \text{eKABs}, which are both based on
compilation schemes \cite{Compilation_Schemes_2000} of \text{eKABs} into PDDL programs.
%
Such a compilation scheme handles the planning domain \emph{independently} of the planning problem.
The compilation schemes presented in~\cite{CalvaneseMPS2016}
and~\cite{BorgwardtHKKNS2022} only support limited families of DLs, namely
\emph{FO-rewritable DLs} in the first case and \emph{Horn DLs} in the second.
%
%
The approach in~\cite{BorgwardtHKKNS2022} works by translating ontology axioms into
derivation rules.
%
Since derivation rules have a single atom in the head, it is not clear
how such a translation would work for logics that are not Horn,
and likely it is impossible.

Our technique to planning with OMPSs is also based on rewritings into PDDL.
To support full OWL DL, and thus also non-Horn description logics, we deviate
however from the idea of compilation schemes.
Namely, our rewriting uses the planning domain together with the planning problem,
and thus solves already part of the planning problem,
which has the additional advantage that there is less work for
the planner to do afterwards.

The main idea is to instantiate all fluents and queries.
Fix a OMPS
$\OP=\tup{\PDDLSpec,\Ostatic,\FluentInter,\QueryInter}$. Roughly, we transform the
OMPS into a PDDL planning specification by adding to
$\PDDLSpec$ a derivation rule for each query predicate. This rule
considers all candidates for the query and specifies for each when
the query would become true.
%
%

%

We define the set $\Fluents$
of \emph{fluents} of $\OP$ as the set of all assertions $A(a)$ and $r(a,b)$ 
that occur in the image of $\FluentInter$. The fluents are the
only atoms in the ontology perspective that can be changed by actions.
Correspondingly, we can construct for any axiom $\alpha$ a formula that is
satisfied in the PDDL perspective of a state iff $\alpha$ is entailed by
the DL perspective:
\begin{align}
 \detect{\alpha}=\bigvee\Big\{\bigwedge_{\beta\in F}\FluentInter^-(\beta)\mid
    F\subseteq\Fluents, \Ostatic\cup F\models\alpha\Big\}\label{eq:detect}
\end{align}

For a given
query specification $q$, we denote by $\cand(q)$ its \emph{candidates}, i.e.
the vector of constants that are compatible with the type specification.
For a query specification $q$ as in \Cref{def:interfaces}, we construct the
corresponding derivation rule $\dr(q)$ as follows:
\begin{align}
 p(x_1,\ldots,x_n)\leftarrow\bigvee_{\tup{c_1,\ldots,c_n}\in\cand(q)}
 \left(
 \bigwedge_{i=1}^n x_i=c_i\wedge
 \bigwedge_{\alpha\in\Phi(\vec{c})}\detect{\alpha}\right)
\end{align}

In our running example, we
would have the following rule:
\begin{align*}
	\pddl{fullHands}(x)\leftarrow (&x=\pddl{stackBot}\wedge\\
	\big(&(\pddl{holds}(x,\pddl{blockA})\wedge\pddl{holds}(x,\pddl{blockB}))\\
	\vee&(\pddl{holds}(x,\pddl{blockA})\wedge\pddl{holds}(x,\pddl{blockC}))\\
	\vee&(\pddl{holds}(x,\pddl{blockB})\wedge\pddl{holds}(x,\pddl{blockC}))\big)
\end{align*}

We also use a rule $\pddl{Inc}\leftarrow\detect{\top\sqsubseteq\bot}$,
where $\pddl{Inc}$ is a fresh nullary derived predicate,
to
detect states that are inconsistent from the ontology perspective.
The rewriting $\rew(\OP)$ of the planning problem is now obtained
from $\PDDLSpec$ by adding $\neg\pddl{Inc}$ as a conjunct to the goal and to the precondition of every
action, and to the set of derivation rules the rule for $\pddl{Inc}$, as well as $\dr(q)$ for all $q\in\QueryInter$.
\begin{theorem}
 Let $\OP$ be an OMPS. Then, every plan for $\OP$ can be translated into a plan
 in $\rew(\OP)$, and every plan in $\rew(\OP)$ can be translated into a plan
 in $\OP$, by replacing each action by the corresponding action in the other
 specification.
\end{theorem}
%
The feasiblity of our technique crucially depends on an efficient computation of $\detect{\alpha}$,
which we discuss in the next section.

\section{Optimized Generation of Justifications}
Since the axioms in the static
ontology $\Ostatic$ are always part of the DL perspective of a state, it makes no sense to
try to modify them with an action.
In the following, we thus assume $\Ostatic\cap \Fluents=\emptyset$.

Let $\Queries$ be the set of assertions for which we need to compute $\detect{\alpha}$.
Due to the monotonicity of entailment, it is sufficient to consider
those subsets $F\subseteq\Fluents$ in~\Cref{eq:detect} that are subset minimal.
Thus, to obtain $\detect{\alpha}$ we need to compute all pairs $\langle \alpha,F \rangle \subseteq \Queries \times 2^\Fluents$ s.t. $\Ostatic \cup F \models\alpha$ and $F$ is minimal, where $\Ostatic$ is the static ontology.
For this, we reduce our problem to the problem of finding minimal inconsistent subsets, also called \emph{justifications}~\cite{HST_Algorithm_OWL}.

\begin{definition}[Justification]
	Given an ontology $\Omc$,
	a subset of axioms $\Just \subseteq \Omc$ is a \emph{justification} iff $\Just \models \top \sqsubseteq \bot$ and for all $\Just' \subsetneq \Just: \Just' \not\models \top \sqsubseteq \bot$.
	We denote the set of all justifications by $\AllJust(\Omc)$.
\end{definition}

%

\noindent
Indeed, justifications can be used to compute what we need:
\begin{restatable}{lemma}{LemmaExplanationJustification}
	\label{lemma:explanationByJustification}
	For sets $\Omc$ and $\Fluents$ of axioms s.t. $\Omc\cap\Fluents=\emptyset$ and some axiom $\alpha$, the following two sets are identical:
	\begin{enumerate}
		\item $\{F \subseteq \Fluents \mid \Omc \cup F \models \alpha, F \textnormal{ is minimal}\}$
		\item $\{J \cap (\Fluents \setminus \{ \neg \alpha \}) \mid J \in \AllJust(\Omc \cup \Fluents \cup \{\neg \alpha\})\}$.
	\end{enumerate}
\end{restatable}

Many efficient algorithms for computing $\AllJust$ have been investigated in the
literature~\cite{PinpointingSAT,JustificationsResolution,PinpointingASP}.
In order  to support full OWL-DL, we use a \emph{black box} algorithm,
which uses a reasoner as a black box, and can thus be used with any existing reasoning system.
The following observations are specific to our problem, and help us to develop dedicated optimizations:
%
\begin{enumerate}[label=\textbf{O\arabic*}]
 \item\label{obs:inc} We need to generate justifications for \emph{inconsistency} as well as for many
	assertions at the same time.
 \item\label{obs:ass} Many \emph{axioms} for which we compute justifications are the same modulo renaming
	of individual names.
 \item\label{obs:just} Many of the \emph{justifications} will be the same modulo renaming of individual names.
\end{enumerate}

\subsection{Basic Algorithm}
\label{sec:basicAlgoritm}
Our basic algorithm is based on the Hitting-Set-Tree algorithm proposed by Reiter~\cite{HST_Algorithm_1987}, which can also be used to generate the set of all justifications for DL ontologies~\cite{HST_Algorithm_OWL}.

The algorithm assumes a method $\SingleJust(\Omc)$ that returns an arbitrary justification for $\Omc$. There are standard implementations available,
i.e. as part of the OWL API~\cite{OWL_API}, that compute $\SingleJust(\Omc)$ using
a black box approach.
The algorithm builds a \emph{hitting-set tree} (HST) where each node is labeled with a justification, and edges are labeled with axioms.
The HST has the property that, if a path from the root is labeled with axioms $\Amc$ and ends on a node labeled with a justification $J$, then
$J$ is a justification for $\Omc\setminus\Amc$. To compute the HST, we use a recursive algorithm that starts with an arbitrary justification $J$ for $\Omc$, and then generates successor nodes by considering all possibilities of removing an axiom $\alpha\in J$ from $\Omc$, so that eventually all justifications are found. More details on the algorithm and a proof why this algorithm works in general for DL ontologies can be found in \cite{HST_Algorithm_OWL}.

\begin{figure}[t]
	\input{hst-algorithm-pseudocode}
	\caption{Basic Hitting-Set-Tree algorithm}
	\label{algorithm-HST}
\end{figure}

The general algorithm is depicted in Figure~\ref{algorithm-HST}. For simplicity, we omit the common optimization of cutting branches, if their path condition is a superset of an already completed path, which we still use in
our implementation.
%
The main function \textsc{ComputeHST} builds the hitting-set tree given an ontology $\Ostatic$, a set of fluent axioms $\Fluents$, the set of all already found justifications $\JustSet$ and the content of the current path leading to the node. In lines 7--9, the algorithm checks if we already found a fitting justification, to reduce the number of calls to $\SingleJust$. Otherwise, we use $\SingleJust$ in Line~11. The function \textsc{Successors} generates the set of axioms to branch on. This function will be
changed in the optimized versions of the algorithm.
%
While the algorithm in~\cite{HST_Algorithm_OWL} branches on all axioms in the justification, we only branch on the fluent axioms. This is sufficient since we do not really need all justifications, but only the ones that differ in their intersections with~$\Fluents$.

\begin{theorem}
	\label{theorem:basicHST}
	Given an ontology $\Omc$ and a set $\Fluents$ of axioms, \textsc{AllJustifications}($\Omc \cup \Fluents$, $\Fluents$) computes all minimal subsets $F\subseteq\Fluents$ s.t. $\Omc\cup F\models\bot$.
	%
\end{theorem}
\begin{proof}
		The theorem follows directly from the original proof by Reiter \cite[Theorems 4.4 and 4.8]{HST_Algorithm_1987}, where our set of fluents $\Fluents$ corresponds to the set of \enquote{components} in Reiter's proof.
	\end{proof}

It follows from \Cref{theorem:basicHST} and \Cref{lemma:explanationByJustification} that we can use this algorithm to compute all required pairs $\tup{\alpha,F}$ and thus $\detect{\alpha}$.

\subsection{Concept-Based Algorithm}
\label{sec:conceptAlgoritm}

Our first optimization of the basic algorithm considers
Observations~\ref{obs:inc} and~\ref{obs:ass}:
If we compute the justifications of $\Ostatic\cup\Fluents\cup\{\neg\alpha\}$ for each $\alpha$ separately, we will
compute all justifications for $\Ostatic\cup\Fluents$ several times.
%
We avoid these redundant computations with the \emph{concept-based algorithm}, which generates one HST that contains the justifications needed for all queries of the same form. For simplicity,
we assume that every $\alpha\in\Queries$ is a concept assertion, ie. an axiom of the form $C(a)$. This is
possible wlog since in OWL DL, every ABox axiom can be equivalently expressed as concept assertion.
%
%

Fix a concept $C$ and some $\SomeIndividuals \subseteq \NI$ so that $\Queries_C = \{C(a) \mid a \in \SomeIndividuals\}$ is a set of assertions for which we want to compute justifications.
We introduce a fresh concept name $A_C$,
and add to the ontology the axiom $C \sqcap A_C \sqsubseteq \bot$ and the assertions $\Queries_C' = \{A_C(a) \mid a \in \SomeIndividuals \}$.

The following lemma shows that this ontology can be used to generate the required sets of fluent.

\begin{restatable}{lemma}{explanationBatch}
		\label{lemma:explanationByJustificationBatch}
		For $\Omc' = \Ostatic \cup \Fluents \cup \{ C \sqcap A_C \sqsubseteq \bot \} \cup \Queries_C'$, the following
		are identical:
	\begin{enumerate}
		\item $\big\{\langle \alpha, F \rangle  \subseteq \Queries_C \times 2^\Fluents \mid \Ostatic \cup F \models \alpha, F \textnormal{ is minimal} \big\}$
		\item $\big\{\langle C(a), J \cap \Fluents \rangle \mid J \in \AllJust(\Omc', \Fluents)$, and

		\hphantom{$\big\{\langle C(a), J \cap \Fluents \rangle \mid$}
		\hphantom{$\big\{\langle C(a), J \cap \Fluents$}
		$A_C(a) \in J$ or $J\cap\Queries_C' = \emptyset\big\}$.
	\end{enumerate}
\end{restatable}

To compute those justifications, we run the algorithm $\textsc{AllJustifications}(\Omc', \Fluents \cup \Queries_C')$ with $\Omc'$ as in the lemma.
We need to include the additional assertions $A_C(a)$ as axioms to branch on, as we are also interested in justifications that differ only in the query that they entail. This is necessary because one combination of fluent axioms might entail several queries and in order to find all of them, we also need to branch by the query that is entailed, not only by the fluent axioms.


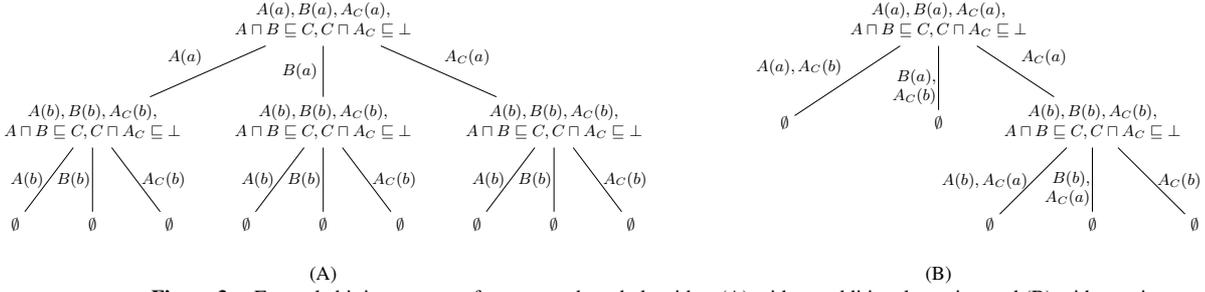
\begin{figure*}[tb]
{\centering
		\resizebox{0.9\textwidth}{!}{
			\begin{tikzpicture}
	\coordinate (originConcept) at (0,0);
	\coordinate (originPruned) at (12, 0);
	
	\node at (originConcept) {\exampleInc{a}}[sibling distance = 4.5cm, level distance=2cm]
	child {node {\exampleInc{b}}[sibling distance = 1.5cm] 
		child {node{$\emptyset$}
			edge from parent node[left] {$A(b)$}}	
		child {node {$\emptyset$}
			edge from parent node[left, xshift=2pt] {$B(b)$}}
		child {node{$\emptyset$}
			edge from parent node[right] {$A_C(b)$}}
		edge from parent node[above left] {$A(a)$}
	}
	child {node {\exampleInc{b}}[sibling distance = 1.5cm]
		child {node{$\emptyset$}
			edge from parent node[left] {$A(b)$}}	
		child {node {$\emptyset$}
			edge from parent node[left, xshift=2pt] {$B(b)$}}
		child {node{$\emptyset$}
			edge from parent node[right] {$A_C(b)$}}
		edge from parent node[left] {$B(a)$}
	}
	child {node {\exampleInc{b}}[sibling distance = 1.5cm]
		child {node{$\emptyset$}
			edge from parent node[left] {$A(b)$}}	
		child {node {$\emptyset$}
			edge from parent node[left, xshift=2pt] {$B(b)$}}
		child {node{$\emptyset$}
			edge from parent node[right] {$A_C(b)$}}
		edge from parent node[above right] {$A_C(a)$}
	};
	
	\node at (originPruned) {\exampleInc{a}}[sibling distance = 3cm, level distance = 2cm]
	child {node {$\emptyset$}
		edge from parent node[above left] {$A(a), A_C(b)$}
	}
	child {node {$\emptyset$}
		edge from parent node[left, xshift=8pt, yshift = -5pt] {\begin{tabular}{r}
				$B(a),$\\
				$A_C(b)$
		\end{tabular}}
	}
	child {node {\exampleInc{b}}[sibling distance = 2cm]
		child {node{$\emptyset$}
			edge from parent node[left] {$A(b), A_C(a)$}}	
		child {node {$\emptyset$}
			edge from parent node[left, xshift=8pt, yshift = -5pt] {\begin{tabular}{r}
					$B(b),$\\
					$ A_C(a)$
				\end{tabular}}}
		child {node{$\emptyset$}
			edge from parent node[right] {$A_C(b)$}}
		edge from parent node[above right] {$A_C(a)$}
	};
	
	\node at ($(originConcept) + (0, -5)$) {\large (A)};
	\node at ($(originPruned) + (0, -5)$) {\large (B)};
\end{tikzpicture}}
}
	\vspace{-5pt}
	\caption{Example hitting-set trees for concept-based algorithm (A) without additional pruning and (B) with pruning.}
	\label{fig:HSTConceptBased}
\end{figure*}

\begin{example}
	\label{example:HSTConcept}
	Consider the static ontology $\Ostatic = \{A \sqcap B \sqsubseteq C\}$, fluent axioms $\Fluents = \{A(a),B(a), A(b), B(b)\}$ and queries $\Queries = \{ \{C(a), C(b)\}$. Then, we add the axioms from $\Queries'= \{A_C(a), A_C(b)\}$ and the axiom $(C \sqcap A_C \sqsubseteq \bot)$, before we run the justification algorithm. The result is the tree in Figure~\ref{fig:HSTConceptBased}~(A). Overall, there is one justification for each query, i.e. the generated tuples are $\langle C(a), \{A(a), B(a)\} \rangle$ and $\langle C(b), \{A(b), B(b)\} \rangle$
\end{example}

The tree in the example (\Cref{fig:HSTConceptBased}) contains three nodes with the same justification. In general, the number of nodes in the tree can be exponential in the number of justifications. To reduce the size of the tree, we modify the HST algorithm to prune branches, while still ensuring that there is at least one node for each existing justification. We do this by modifying the function \textsc{Successors} as shown in \Cref{algorithm:conceptBased}. The idea is to have two kinds of sub trees: (i)~sub trees where we search for other justifications of the same query and (ii)~sub trees where we search for justifications of other queries. We achieve this by removing the introduced concept assertions for all individuals whose query is not explained in the node for all branches except the one that is already labeled with the introduced concept assertion. Overall, this results in a tree that has a chain-structure of connected smaller trees, one for each query.

\begin{figure}[tb]
	\begin{algorithmic}[1]
		\Function{Successors}{$\Just, \Fluents, \Branches$}
		\If{$ (C \sqcap A_C \sqsubseteq \bot) \not \in J$}
			\State $\Branches \leftarrow \{ \{B\} \mid B \in \Just \cap \Fluents\}$
		\Else
			\State $a\leftarrow a \in \NI$ s.t. $A_C(a) \in J$
			\State $\Branches \leftarrow \{\{ A_C(a) \}\}$
			\For{$\alpha\in \Just \cap (\Fluents \setminus \Queries_C')$}
			\State $B \leftarrow \{\alpha \} \cup \Queries_C' \setminus \{A_C(a)\}$
			\State $\Branches \leftarrow \Branches \cup \{ B\}$
			\EndFor
		\EndIf
		
		\EndFunction
	\end{algorithmic}
	\caption{Optimized branching for concept-based algorithm.}
	\label{algorithm:conceptBased}
\end{figure}


\begin{example}
	Consider the ontology, fluents and queries from Example~\ref{example:HSTConcept}. The resulting tree is depicted in Figure~\ref{fig:HSTConceptBased}~(B). Note, that there is now only one sub-tree that contains the justification for the individual $b$. Overall, the tree contains only 7 nodes whereas the tree without the pruning requires 13 nodes.
\end{example}

\begin{restatable}{theorem}{TheoConceptCorrectness}
	\label{theorem:ConceptCorrectness}
	Given an ontology $\Omc$, a set of fluent axioms $\Fluents$ and a set of queries $\Queries$, the concept-based algorithm computes
	all minimal subsets $F\subseteq\Fluents$ s.t. $\Omc\cup F\models\bot$.
\end{restatable}
	
\subsection{Schema-Leveraging Algorithm}
\label{sec:schemaAlgoritm}
Looking at the generated HSTs, e.g. in \Cref{fig:HSTConceptBased}, we observe that many justifications are structurally the same and only differ in the individuals used in the ABox axioms (Observation~\ref{obs:just}). This is especially undesired in cases where there are many individuals that occur in the same Abox axioms. This can e.g. happen if the individuals describe waypoints that can all be occupied by all mobile objects.
Calling \SingleJust is relatively costly, since each time it has to explore the space of subsets of the current ontology, and check their consistency using calls to the black box reasoner.
We can thus improve the run time of our algorithm significantly if we can reduce the number of calls to
\SingleJust.


We fix a set $\Xbf$ of \emph{variables}.
A \emph{valuation} is a partial function $\val:\Xbf\nrightarrow\NI$. An \emph{axiom pattern} is an axiom where some individual names may be replaced by variables. Given an axiom pattern $\AxPattern$, $\val(\AxPattern)$ denotes the set of axiom patterns obtained from \AxPattern by replacing each variable $x$ for which $\val(x)$ is defined by $\val(x)$.

\begin{definition}
	A \emph{justification schema} is a pair $\tup{\Jscheme,\Val}$ s.t. $\Jscheme$ is a set of axiom patterns, $\Val$ is a set of valuations, and for each $\val\in\Val$, $\val(\Jscheme)$ is a subset-minimal set s.t. $\val(\Jscheme)\subseteq\Omc$ and $\val(\Jscheme)\models\top \sqsubseteq \bot$. 
	We then call $\val(\Jscheme)$ an \emph{instantiation of $\tup{\Jscheme,\Val}$}.
\end{definition}

Whenever a justification $\Just$ is found, we construct the corresponding axiom patterns \AxPattern by replacing every individual by a distinct variable in the axioms of the justification. After that, we identify the associated justification schema $\tup{\Jscheme,\Val}$, i.e. we search in the ontology ($\Ostatic \cup \Fluents \cup \Queries_C'$) for fitting axioms to compute all valid valuations $\val$, taking into account also the static type of the query. This check is purely syntactic and therefore much faster than computing the justifications
with a reasoner.
%
%
Using the justification schema, we compute all justifications that follow the same schema by computing $\val(\Jscheme)$ for all valuations $\val$ in the schema.
The found justification are then added to the set of all found justifications.
This change in the algorithm does not affect the correctness result from Theorem~\ref{theorem:ConceptCorrectness} as it does not affect the structure of the hitting-set tree.

\section{Evaluation}
\subsection{Implementation}
We implemented all the tailored hitting-set-tree algorithms.\footnote{The implementation and scripts to reproduce the evaluation can be found online~ \cite{ZenodoArtifact2024}.} We used the reasoner system HermiT \cite{GlimmHMSW2014} for our implementation but it is in general agnostic to which reasoner is used. 


We add the generated derivation rules to the otherwise unchanged PDDL domains. We used the fast-downward planning system \cite{Helmert2006} with the heuristic A* for planning. Although not the most advanced heuristic, this is, as far as we know, the best that can handle the structure of the derivation rules as generated by our algorithm.

\subsection{Experiments}
\paragraph{Benchmark Sets}
As we are the first to present a planning approach that works with OWL-DL ontologies, finding benchmark sets that use ontologies based on OWL-DL is difficult. In total, we gathered 116 planning problems from five different domains. Three benchmark domains (drones, queens and robotConj) come from existing work on planning with Horn-DLs \cite{BorgwardtHKKNS2022}. We choose from this work the benchmark sets with the the most expressive TBox axioms, namely the ones that go beyond DL-Lite, as we are looking for ontologies based on expressive DLs. 

The two other benchmark sets were created by ourselves. The domain \enquote{blocksworld} is adapted from the blocksworld domain as described in Example~\ref{example:overview}. The ontology describes the parts and capabilities of a robot interacting with the blocks. The axioms in this ontology go beyond what can be handled by other existing approaches. The domain \enquote{pipes} describes an inspection task for an autonomous underwater robot. The planning domain describes a grid of waypoints and actions of moving and interacting with valves. The ontology describes the objects, their location at the waypoints and their relations, e.g. connections between pipe sections. 

\begin{table}[tb]
	\caption{Axiom numbers and query candidates in the benchmarks.}
%
	\label{table:statistics}
	\resizebox{\columnwidth}{!}{
	\begin{tabular}{l r r r   r r r r}
		\textbf{domain} &  \multicolumn{2}{c}{\textbf{ ontology}} & \multicolumn{1}{c}{\textbf{TBox }} & \multicolumn{2}{c}{\textbf{ fluent}} & \multicolumn{2}{c}{\textbf{query candidates}} \\
		& min & max & & min & max & min & max\\
		drones & 168 & 622 & 17 &  52 & 202 & 1,300 & 20,200 \\
		queens & 118 & 463 & 17 & 25 & 100 & 675 & 10,200 \\
		robotConj & 40 & 306 & 30--147 & 20& 134 & 20 & 134 \\
		blocksworld & 11 & 17 & 5 & 4 & 10 & 1 & 1\\
		pipes & 118 & 416 &43 & 33 & 207 & 16 & 111 
	\end{tabular}
	}
\end{table}

As shown in Table~\ref{table:statistics}, the benchmark domains have different characteristics. While most rely on an ontology with few TBox axioms, the domains \enquote{pipes} and \enquote{robotConj} have a more detailed TBox. Also, the number of query candidates for which we search for justifications varies a lot, e.g. \enquote{drones}  and \enquote{queens}  have thousands of them while \enquote{blocksworld} has only one.

\paragraph{Setup}
We used an overall time limit of 30min for the plan computation and
used a memory limit of 8GB for our experiments. We ran the experiment on a server with an Intel Xeon Gold 5118 CPU @2.30GHz running Ubuntu~20.04. We consider four planning
algorithms. The algorithms \textsc{Basic}, \textsc{Concept} and \textsc{Schema}
follow the procedure described in this paper and use the respective algorithms
from \Cref{sec:basicAlgoritm}, \Cref{sec:conceptAlgoritm} and
\Cref{sec:schemaAlgoritm}. The algorithm \textsc{Horn} is the best performing
approach to do ontology-mediated planning as of
today~\cite{BorgwardtHKKNS2022}, but only supports the problems with Horn
ontologies.
%

\paragraph{Results}

\begin{table}[tb]
	\caption{Evaluation of different planning algorithms on different domains.}
	\label{table:solvedInstances}
	\centering
	\resizebox{\columnwidth}{!}{
	\begin{tabular}{lrrrrr }
		\textbf{domain} & \textbf{\#instances} & \multicolumn{4}{c}{\textbf{\#solved instances}} \\
		&&  \textsc{Basic} & \textsc{Concept} & \textsc{Schema} & \textsc{Horn}  \\[0.2cm]		
		{drones} & 24   & 0 & 9 &16 & \textbf{24}\\
		{queens} & 30 & 11 & \textbf{25} & \textbf{25} & 23 \\
		{robotConj} & 20  & 4 & 4& 4 & \textbf{20}\\
		{blocksworld} & 21 & \textbf{14} & \textbf{14} & \textbf{14} & --- \\
		{pipes} & 21  & 17 & \textbf{19}& \textbf{19} & 14 \\

	\end{tabular}
}
\end{table}

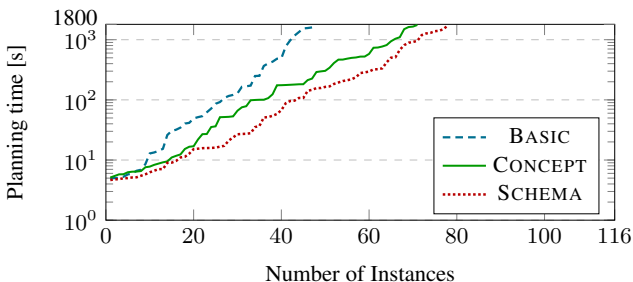
\begin{figure}[tb]
	\resizebox{\columnwidth}{!}{

\begin{tikzpicture}
	\begin{axis}[
		width=\columnwidth,
		height=0.5\columnwidth,
		title={},
		xlabel={Number of Instances},
		ylabel={Planning time [s]},
		ymode=log,
		xmin=0, xmax=116,
		ymin=1, ymax=1800,
		xtick={0,20,40,60,80,100, 116},
		ytick={1,10, 100, 1000},
		extra y ticks={1800},
		extra y tick style={
			grid=minor,
			yticklabel={1800},
			yticklabel style={yshift=0.7ex, anchor=east}},
		legend pos=south east,
		ymajorgrids=true,
		grid style=dashed,
		]

		\addplot[
		color=MidnightBlue,
		line width=0.9pt,
		densely dashed
		]
		 table [
		x=instance, 
		y=basic, 
		col sep=comma
		] {plotData.csv};
		
		\addplot[
		color=OliveGreen,
		line width=0.8pt
		] 
		table [
		x=instance, 
		y=concept,
		col sep=comma
		] {plotData.csv};
		
		\addplot[
		color=BrickRed,
		line width=1pt,
		densely dotted
		] 
		table [
		x=instance, 
		y=schema, 
		col sep=comma
		] {plotData.csv};

		\legend{\textsc{Basic}, \textsc{Concept}, \textsc{Schema}}
		
	\end{axis}
\end{tikzpicture}
	}
	\vspace{-15pt}
	\caption{Cactus plot showing how many instances could be solved by the different algorithms within a given time bound. The y-axis is logarithmic.}
	\label{graph:time}
\end{figure}

Table~\ref{table:solvedInstances} shows how many instances of the different domains could be solved by the different algorithms. The two introduced optimizations of the basic algorithm impact the number of solved instances. 
The optimizations are most relevant for the benchmark domains \enquote{drones} and \enquote{queens} while there is little gain on the other benchmark domains. For two of them (\enquote{blocksworld} and \enquote{pipes}) the planner cannot solve the generated planning problem in time but the justifications for all instances are generated by \textsc{Concept} and \textsc{Schema} and the optimizations are not necessary.
The remaining benchmark domain (\enquote{robotConj}) is especially hard to reason about and neither of our optimizations helps with that. The main problem with this domain is that there are many inconsistent fluent combinations, which result in many additional nodes in the HSTs.

We do not see a correlation between the number of solved instances and the complexity of the ontology axioms. Notably, the justifications for the benchmark domain \enquote{blocksworld} could always be generated in under 10s. The number of queries affects the performance, as expected, but our evaluation also shows that our optimized algorithms can solve instances with thousands of queries, e.g. in benchmark \enquote{queens}.

To compare the impact of our optimizations in more detail, Figure~\ref{graph:time} shows a cactus plot with the number of instances---considering all domains---and the time within they could be solved. We see that there are huge differences between the approaches (note that the y-axis is logarithmic). Especially interesting is that the curve of \textsc{Schema} is consistently lower that the curve of \textsc{Concept}. This shows that the additional effort of computing the justification schemata is well spend, even on simple instances.

When comparing our algorithms with \textsc{Horn}, we see that our best-performing algorithm \textsc{Schema} is able to solve more instances for some domains while \textsc{Horn} can solve more instances in other domains. Since \textsc{Horn} only supports a fragment of OWL DL,
while we support all of it, better performance of \textsc{Horn} is to be expected on some cases,
since we did not optimize for the peculiar aspects of
Horn DLs.
%
%
%
In general, the reasoning phase is slower for our algorithm because we do reasoning on the individuals of the ontology, while \textsc{Horn} only rewrites the axioms of the TBox into derivation rules in PDDL. On the other side, the structure of our generated rules allow for faster execution of the PDDL planner, e.g. for the domain \enquote{pipes} the median number of states that was evaluated per second was 9,000 for \textsc{Horn} but 140,000 for \textsc{Schema}. This leads to shorter times of the planning phases for our algorithms compared to \textsc{Horn}. A more detailed comparison of the execution times can be found in the %
\tr{appendix}\submission{extended version of the paper \cite{EXTENDED}}.
Overall, selecting the best planning approach depends on the specific ontology and how difficult it is to find a valid plan, i.e. how many states need to be evaluated.

\section{Conclusion}
We presented a new framework for integrating ontologies into planning, as well as
the first approach to do planning with ontologies based on very expressive ontology languages, namely supporting OWL-DL. Central to our method is an efficient generation of justifications for the DL queries in the planning specification. While this worked well for some examples, there is still room for improvement as our approach does not work well on all domains. Mainly, the handling of inconsistent fluent assignments could be improved. Another problem is obtaining realistic benchmark domains. We are not aware of examples that make use of complex OWL axioms and some of the existing benchmarks seem to be very academic. Hence, developing a better set of benchmark domains should be a future goal.



\begin{ack}
Tobias John is part of the project REMARO that has received funding from the European Union's Horizon 2020 research and innovation programme under the Marie Sk\l{}odowska-Curie grant agreement No~956200.
Patrick Koopmann is partially supported by the German Research Foundation (DFG), grant 389792660 as part of TRR 248 --– CPEC.
\end{ack}



\bibliography{bib-dl-planning}

\noteInline{Page limit: 8 (7 pages content + 1 page for references)}

\tr{\newpage
\appendix
\section{Proofs of Lemmas and Theorems}

\LemmaExplanationJustification*
\begin{proof}
		\enquote{$\subseteq$}: Let $F \subseteq \Fluents$ be a minimal set with $\Omc \cup F \models q$. Then, there is a minimal set $J \subseteq \Omc \cup F$ with $J \models q$. Because $F$ is minimal, it is completely contained in $J$, i.e. $F \subseteq J$. $J \cup \{ \neg q\} \models \top \sqsubseteq \bot$, because $q$ can be inferred from $J$. Hence, $J \cup \{ \neg q\}$ or $J$ are a justification for the ontology $\Omc \cup \Fluents \cup \{\neg q\}$ (because $J$ is minimal). Using the facts that $\neg q \not\in F$, because $\neg q$ can not be part of a minimal explanation of $q$, and $\Omc \cap \Fluents = \emptyset$ it follows that $F = J \cap (\Fluents \setminus \{ \neg q \})$.

		\noindent
		\enquote{$\supseteq$}: Let $J \in \AllJust(\Omc \cup \Fluents \cup \{\neg q\})$ be a justification. Let $J' = J \setminus \{ \neg q\}$ and we show that $J' \models q$. If $\{\neg q\} \not \in J$, we know that $J'$ is inconsistent and models everything, including $q$. If $\{\neg q\} \in J$, we know that $J'$ is consistent as all justifications are minimal. Hence, we can infer $q$ from $J'$ as this is the only reason for an inconsistency together with $\neg q$. We split $J'= \Omc' \cup F$ such that $\Omc' \subseteq \Omc$ and $F \subseteq \Fluents$. Because $J' \models q$, we infer that $\Omc' \cup F \subseteq \Omc \cup F \models q$.
\end{proof}

\TheoConceptCorrectness*
\begin{proof}
   Correctness of the algorithm is trivial as a black-box explanation
generator is used to generate the justifications.

   The argumentation for completeness needs more work. According to
\Cref{theorem:basicHST}, the algorithm works correctly, if it only
branches by fluent and query axioms independently. We need to show, that the
additional restrictions, i.e. the $A_C(a)$ axioms on the branches do not omit
valid justifications. To do so, we fix a node $v$ and show that each of the
justifications occurring in one of the child nodes in the tree of the original
algorithm can also occur in a child node in the tree of the concept-based
algorithm. Assume that the node $v$ contains the axiom $C \sqcap A_C
\sqsubseteq \bot$. Otherwise, the labels on the branches are the same as for the
basic algorithm, so the algorithm is trivially correct. Then, the node $v$
contains exactly one assertion of the form $A_C(a)$ (at least one, because
otherwise $C \sqcap A_C \sqsubseteq \bot$ is not required; at most one,
because of the minimality of justifications.) Every justification in a child
that does not contain $A_C(a)$ can occur in the tree below the branch labeled
with $A_C(a)$. Every other child contains $A_C(a)$ and therefore no other
$A_C(b)$ (because of the minimality of justifications). In the tree of the
basic algorithm, there is a fluent axiom $\alpha$ that is a label of the
branch to the sub-tree that contains the child. In the tree of the
concept-algorithm, the child node can occur in the sub-tree where the
branch is
labeled with $\{f\} \cup \{A_C(b) | b \in \NI, b\neq a \}$.
\end{proof}

\explanationBatch*
\begin{proof}
	\enquote{$\subseteq$}: Let $\langle Q(i), F \rangle$ with $\Omc \cup F \models Q(i)$ such that $F$ is minimal. Then, there is a minimal set $J \subseteq \Omc \cup F$ with $J \models Q(i)$. Because $F$ is minimal, it is completely contained in $J$, i.e. $F \subseteq J$. As $Q(i)$ can be inferred from $J$, the set $J' = J \cup \{Q \sqcap Q' \sqsubseteq \bot,  Q'(i)\}$ is inconsistent. Hence, one of the sets $J$, or $J'$ is a justification for the inconsistency of $O'$ (because $J$ is minimal). The sets $J \cup \{Q'(i)\}$ and $J \cup \{Q \sqcap Q' \sqsubseteq \bot\}$ can not be such justifications because the concept $Q'$ only occurs in the axiom $Q \sqcap Q' \sqsubseteq \bot$ and $Q'$ does not occur in $\Omc$ or $F$. Hence, if $J \cup \{Q'(i)\}$ or $J \cup \{Q \sqcap Q' \sqsubseteq \bot\}$ is inconsistent, so is $J$.
	No matter if $J$ or $J'$ is a justification, they are relevant for the second set and $J \cup \Fluents = J' \cup \Fluents = F$. Hence, $\langle Q(i), F \rangle$ is contained in the second set.	
	
	\noindent
	\enquote{$\supseteq$}: Let $J \in \AllJust(\Omc')$ be a justification. We distinguish two cases. (i) If $J \cup \Queries' = \emptyset$, the pair $\langle Q(i), J \cap \Fluents\rangle$ is in the set for all $Q(i) \in \Queries$. Because $J$ is minimal, contains no assertion from $\Queries'$ and the concept $Q'$ occurs nowhere else, we know that $\{ Q \sqcap Q' \sqsubseteq \bot\} \not \in J$. Hence, $J \subseteq \Omc \cup F$. Let $F = (J \cap \Fluents)$ be the set of all fluent axioms in $J$. As $J$ is inconsistent, so is $\Omc \cup F$ and therefore $\Omc \cup F \models Q(i)$ for all $Q(i)$. 
	(ii) If there is a $Q'(i)$ with $Q'(i) \in J \cup \Queries'\}$, we know that $J' = J \setminus \{Q \sqcap Q' \sqsubseteq \bot\}$ must be consistent as $J$ is minimal. Therefore, $J' \models Q(i)$ as this is the only possibility that $J$ is inconsistent. As the concept $Q'$ occurs in neither $\Omc$ nor $\Fluents$, the assertion $Q'(i)$ is not be necessary to infer $Q(i)$ in $J'$ and we infer that $J' \cap (\Omc \cup \Fluents) \models Q(i)$. Therefore, there is a set $F = J' \cup \Fluents = J \cup \Fluents$ with $\Omc \cup F \models Q(i)$. 
\end{proof}

\section{Further Plots}

\begin{figure*}[tbh]
	\resizebox{0.32\textwidth}{!}{
		\begin{tikzpicture}
	\begin{axis}[
		width=0.75\columnwidth,
		height=0.75\columnwidth,
		title={Reasoning time},
		xlabel={Time \textsc{Horn} [s]},
		ylabel={Time \textsc{Schema} [s]},
		ymode=log,
		xmode=log,
		xmin=0.1, xmax=1800,
		ymin=0.1, ymax=1800,
		xtick={0.1,1,10,100, 1000},
		ytick={0.1,1,10, 100, 1000},
		extra y ticks={1800},
		extra y tick style={
			grid=minor,
			yticklabel={timeout},
			yticklabel style={yshift=0.7ex, anchor=east}},
		extra x ticks={1800},
		extra x tick style={
			grid=minor,
			xticklabel={timeout},
			xticklabel style={xshift=0ex, anchor=north west}},
		legend pos=south east,
		ymajorgrids=true,
		xmajorgrids=true,
		grid style=dashed,
		]

		\addplot[
		color=MidnightBlue,
		opacity=0.4,
		mark=*,
		only marks,
		]
		table [
		x=reasoningHorn, 
		y=reasoningSchema, 
		col sep=comma
		] {planningTimes.csv};
		
		\addplot [draw=gray,
		dashed] coordinates {
			(0.1,0.1) (1800,1800)
		}; 
		
	\end{axis}
\end{tikzpicture}
	}
	\resizebox{0.32\textwidth}{!}{
		\begin{tikzpicture}
	\begin{axis}[
		width=0.75\columnwidth,
		height=0.75\columnwidth,
		title={Planning time},
		xlabel={Time \textsc{Horn} [s]},
		ylabel={Time \textsc{Schema} [s]},
		ymode=log,
		xmode=log,
		xmin=0.1, xmax=1800,
		ymin=0.1, ymax=1800,
		xtick={0.1,1,10,100, 1000},
		ytick={0.1,1,10, 100, 1000},
		extra y ticks={1800},
		extra y tick style={
			grid=minor,
			yticklabel={timeout},
			yticklabel style={yshift=0.7ex, anchor=east}},
		extra x ticks={1800},
		extra x tick style={
			grid=minor,
			xticklabel={timeout},
			xticklabel style={xshift=0ex, anchor=north west}},
		legend pos=south east,
		ymajorgrids=true,
		xmajorgrids=true,
		grid style=dashed,
		]

		\addplot[
		color=MidnightBlue,
		mark=*,
		only marks,
		opacity=0.4,
		]
		table [
		x=planningHorn, 
		y=planningSchema, 
		col sep=comma
		] {planningTimes.csv};
		
		\addplot [draw=gray,
		dashed] coordinates {
			(0.1,0.1) (1800,1800)
		}; 
		
	\end{axis}
\end{tikzpicture}
	}
	\resizebox{0.32\textwidth}{!}{
		\begin{tikzpicture}
	\begin{axis}[
		width=0.75\columnwidth,
		height=0.75\columnwidth,
		title={Total time},
		xlabel={Time \textsc{Horn} [s]},
		ylabel={Time \textsc{Schema} [s]},
		ymode=log,
		xmode=log,
		xmin=0.8, xmax=1800,
		ymin=0.8, ymax=1800,
		xtick={0.1,1,10,100, 1000},
		ytick={0.1,1,10, 100, 1000},
		extra y ticks={1800},
		extra y tick style={
			grid=minor,
			yticklabel={timeout},
			yticklabel style={yshift=0.7ex, anchor=east}},
		extra x ticks={1800},
		extra x tick style={
			grid=minor,
			xticklabel={timeout},
			xticklabel style={xshift=0ex, anchor=north west}},
		legend pos=south east,
		ymajorgrids=true,
		xmajorgrids=true,
		grid style=dashed,
		]

		\addplot[
		color=MidnightBlue,
		mark=*,
		only marks,
		opacity=0.4,
		]
		table [
		x=totalHorn, 
		y=totalSchema, 
		col sep=comma
		] {planningTimes.csv};
		
		\addplot [draw=gray,
		dashed] coordinates {
			(0.1,0.1) (1800,1800)
		}; 
		
	\end{axis}
\end{tikzpicture}
	}
	\caption{Comparison of run time of \textsc{Schema} and \textsc{Horn}. For the planning times, we only depicted the cases where both approaches where able to finish the reasoning step within the time limit.}
	\label{fig:timePlots}
\end{figure*}
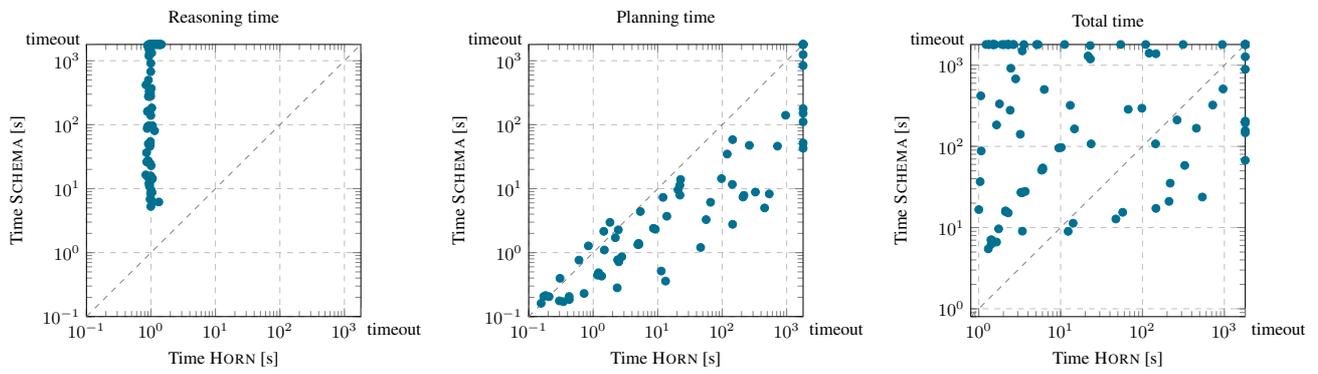

We compare the time needed to solve the planning problems for \textsc{Horn} and \textsc{Schema} in more detail. Figure~\ref{fig:timePlots} shows comparisons for the time used for reasoning, planning and the total time. 

The first plot shows how different the first computation step of the two algorithms is. Because \textsc{Horn} only translates the ontology axioms into derivation rules, the computation time is always short (around 1s). On the other hand, \textsc{Schema} generates all explanations for the queries for the different individuals. Hence, the reasoning needs way more time, especially when there are many individuals. However, this initial effort can pay off for instances where a large portion of the search space needs to be explored. 

The second plot shows, why this is the case: the algorithm \textsc{Schema} results in fewer time spend by the reasoner to find a plan. Depending on the example, the advantage can be by an order of magnitude or even higher. This is due to the structure of the generated derivation rules. While \textsc{Horn} generates rules that hierarchical depend on each other, the rules created by \textsc{Schema} creates rules that only add one layer on top of the planning atoms. Thus, evaluating a state is much faster for \textsc{Schema}. This pays especially off for the instances where many states need to be evaluated, i.e. the hardest planning problems. 

In total, this results in a mixed picture for the total time (depicted in the third plot). While \textsc{Horn} is faster for most instances, there are some instances for which \textsc{Schema} is faster. 
}

\end{document}